\documentclass[sigconf,nonacm]{acmart}
\usepackage{amsmath,amsthm}
\usepackage{booktabs}
\usepackage{algorithm}
\usepackage[noend]{algpseudocode}
\setcopyright{none}

\title[Progressive Risk Vesting]{Spawn Freely, Act Sparingly:
Progressive Risk Vesting for Recursive LLM-Agent Trees}
\hypersetup{pdfauthor={Molly Wang}}

\newtheorem{theorem}{Theorem}
\newtheorem{proposition}{Proposition}
\newtheorem{corollary}{Corollary}

\newcommand{\E}{\mathbb{E}}
\newcommand{\Prb}{\mathbb{P}}
\newcommand{\F}{\mathcal{F}}
\newcommand{\RA}{\mathcal{R}_{\!A}}
\newcommand{\prv}{\textsc{PRV}}

\author{Molly Wang}
\affiliation{%
  \institution{Imperial Business School}
  \city{London}
  \country{United Kingdom}}
\email{jw923@ic.ac.uk}
\begin{document}

\begin{abstract}
Recursive LLM agents can broaden their search by spawning specialists.  Some branches later
request tools that send data or deploy code.  When should a branch receive authority to act?  We
distinguish \emph{sandbox spawning}, in which external controls prevent the specified harm, from
\emph{capability activation}, in which a selected branch crosses an irreversible-action boundary.
Progressive Risk Vesting (\prv{}) holds a trajectory-level risk budget in escrow and debits it as
branches are activated.  We prove an anytime harm bound for adaptively generated trees.  Branch
outcomes may be dependent, but each local certificate needs to remain valid conditional on the
full pre-activation history, including the information used to select the request.  When
activation gates, branch charges, and compute constraints are held fixed, delayed vesting
preserves every policy available under irrevocable spawn charging.  Marginal risk estimates can
still fail after branch selection.  In a stylized branching model, trajectory harm changes as the
authority reproduction number $\RA$ crosses one.  As local risk $p$ approaches zero, trajectory
harm is proportional to $p$ below criticality, proportional to $\sqrt p$ at criticality, and
retains a positive floor above it.  A finite-type occupancy model yields risk and compute shadow
prices.  For nested fanout modes with decreasing marginal value per unit risk, these prices
produce a threshold rule.  Branching calculations and a split-sample experiment illustrate the
results.  These synthetic studies do not estimate safety in deployed agents.  The analysis
suggests a design rule: search broadly in the sandbox and grant recursive authority sparingly,
with an explicit risk charge.
\end{abstract}



\maketitle

\section{Introduction}

\subsection{The decision hidden inside recursive spawning}

Consider a coding agent asked to repair a production service.  It sends children to diagnose the
issue and draft a patch; those children may spawn their own specialists.  Most branches work in
a disposable repository and return candidate artifacts.  At some point, one branch requests a
capability with external effects, perhaps access to a secret or permission to deploy.  The system
has moved from generating possibilities to changing the world.

Recent systems use recursive delegation to create reasoning threads and run recursive workloads
in parallel \cite{schroeder2025thread,song2026webswarm}.  More branches may improve coverage.  They
also consume compute and time and increase the number of branches that may receive
authority.  Each authorized descendant adds another possible path to irreversible harm, whether
through prompt injection or tool failure.  NIST's agent-hijacking evaluations include data
exfiltration and malicious code execution \cite{nist2025hijacking}.  This makes the authorization
chain a separate safety concern \cite{south2025delegation}.

Two possible controls are hard fanout limits and a risk charge for every new child.  Either can
suppress useful exploration.  If external controls prevent a reasoning branch from causing the
specified harm, that branch consumes compute without increasing exposure to the same catastrophe.
Charging it as if it held production credentials can lose the option value of generating and
comparing candidates.  If every child receives the parent's authority, broad search also expands
the attack surface.

We study the difference between \emph{delegating work} and \emph{delegating the right to act}.
A child starts with sandboxed capabilities.  The shared risk budget is charged when a governor
approves a specified irreversible action or a capability that could enable one.  We therefore
distinguish two trees: a sandbox tree that consumes compute and an authority tree that consumes
the risk budget.  Our central question is:

\begin{quote}
\emph{How can broad recursive search coexist with controlled irreversible authority?}
\end{quote}

\subsection{Contributions}

We propose Progressive Risk Vesting (\prv{}).  Our main contributions are:

\begin{enumerate}
  \item \textbf{Risk escrow.}  Branches may explore inside an enforced sandbox.  The governor
  debits the shared risk budget when a selected branch crosses an irreversible boundary.  A
  branch cancelled before external exposure consumes compute but no risk allowance.
  \item \textbf{Anytime guarantee on a random tree.}  If each local certificate remains valid
  conditional on the full pre-activation history, the root escrow bounds the probability of any
  policy-defined catastrophic activation.  The topology and stopping rule may adapt to history,
  and branch outcomes need not be independent.
  \item \textbf{Selection and option value.}  With fixed gates, charges, and compute limits,
  delayed vesting keeps every spawn-charging policy feasible and can leave room for better
  choices.  A counterexample
  also shows that marginal calibration can fail when the parent selectively escalates suspicious
  or difficult branches.
  \item \textbf{Authority phase transition.}  In the unbounded homogeneous model, the authority
  reproduction rate determines whether trajectory harm is linear in local risk, proportional to
  its square root, or bounded below by a nonzero floor.
  \item \textbf{Optimization and numerical evidence.}  A multitype occupancy program gives risk
  and compute shadow prices.  For nested fanout modes with decreasing marginal value per unit
  risk, these prices imply a base-stock fanout rule.  Two reproducible synthetic studies illustrate
  the phase transition and the option value of delayed vesting.
\end{enumerate}

The guarantees depend on their stated conditions and should not be read as universal safety
claims.  The surrounding system needs to enforce the ``sandbox'' and define the catastrophe to
cover the deployment's material harms.  Activation certificates also need to remain valid after
adaptive selection.  The claims apply only when these assumptions hold in deployment.

\section{Related Work}

\paragraph{Recursive systems.}
Several recent systems create agents or reasoning branches at runtime.  THREAD creates reasoning
threads \cite{schroeder2025thread}, while AgentSpawn and AOrchestra construct task-specific agents at
runtime \cite{costa2026agentspawn,ruan2026aorchestra}.  Recursive Agent Harnesses and WebSwarm
expand recursive workloads \cite{lumer2026rah,song2026webswarm}.  MaAS studies the related
problem of allocating inference resources across candidate architectures
\cite{zhang2025maas}.  Recent work also suggests that adding homogeneous agents has diminishing
returns, while diverse information channels may remain useful \cite{yang2026diversity}.  These
studies examine the benefits and compute costs of recursive execution.  We focus on the point
where a branch receives irreversible authority and ask how harm probability
should be budgeted across the full trajectory.

\paragraph{Delegation and compositional risk.}
Existing safety mechanisms often govern how permissions and resources pass down a delegation
chain.  Authenticated delegation studies authorization and accountability
\cite{south2025delegation}.  Safe Bilevel Delegation transfers authority under a probabilistic
constraint, leaving open how that constraint composes across a recursively generated episode
\cite{sun2026sbd}.  Other work narrows delegated scopes
\cite{muruaga2026bounded,ibrahim2026overlay}, and LACUNA rejects invalid whole actions before
they affect the environment \cite{zhao2026lacuna}.  Agent Contracts and related resource-budget
methods conserve resources as agents delegate work or change topology
\cite{ye2026contracts,talluri2026rgao,khan2026token}.  Our ledger applies the same conservation
idea to a shared harm-probability budget.  The extra statistical condition is that the charged
certificate remain valid after the parent has selected the branch.

Risk-Sensitive Agent Compositions optimizes VaR and CVaR over paths in a known agent DAG
\cite{shabadi2026risk}.  Pipeline-aware and post-hoc methods certify fixed modular trajectories
\cite{kotte2026pasc,li2026trajectory}.  Here, the topology unfolds online and the parent selects a
branch before it acts.  We use a conditional certificate at each activation.  When the policy is
fixed in advance, whole-policy certification remains an alternative.

\paragraph{Propagation and stochastic processes.}
Error-cascade models trace an injected mistake through message dependencies
\cite{xie2026spark}.  Recursive spawning creates a separate concern: each authority-bearing
node adds another opportunity for a local irreversible failure.  Classical branching-process
theory describes extinction and total progeny \cite{athreya1972branching}.  We use these results
to separate the reproduction of sandbox compute from the reproduction of authority.  Standard
constrained-control occupancy measures \cite{altman1999cmdp} then provide a way to allocate
compute and risk across authority-bearing branches.

\section{Model}

\subsection{Sandbox tree, authority tree, and activation boundary}

We track sandboxed exploration, the order in which branches receive real-world authority, and a
root risk escrow that descendants cannot copy.

A root agent grows a sandbox tree $\mathcal T^S$ in response to what its branches discover; the
model allows this tree to be countable.  A sandboxed node may reason, read immutable task-local
data, and create further sandboxed nodes.  The catastrophe predicate describes the
application-level harm that the system aims to prevent.  In this paper, a node counts as
sandboxed when external controls make that outcome unreachable with the node's available
capabilities.  A repository copy with no network access is one example.  If the policy treats
reading sensitive data itself as a catastrophe, that read is an activation rather than a sandbox
operation.

\emph{Activation} marks the move from sandboxed reasoning to an action or capability that can
realize the specified catastrophe predicate.  A branch reaches this boundary when it proposes a
specified action or requests a capability scope that could enable one.  A certificate for a
nonrenewable capability lease needs to cover every use and downstream effect during the lease;
otherwise, each use requires a new gate decision.  An approved request enters the activation
sequence when its action or capability is exposed, and we number these requests chronologically by
$i=1,2,\ldots$.  Rejected requests do not enter the sequence.  The governor assigns every approved
request an atomic position in the admission order, although approved branches may later run in
parallel.  Let $H_i$ be the event that activation $i$ causes a policy-defined catastrophe.  We pad
a finite activation sequence with null events and zero allowances.

The history $\F_{i-1}$ contains all observations available to the governor before it approves
activation $i$.  This includes the generated tree, the requested action and scope, branch and test
evidence, and earlier grants and outcomes.  It also includes all information used to select the
request.  The governor
chooses an $\F_{i-1}$-measurable allowance $a_i\geq0$ and requires

\begin{equation}
 \Prb(H_i\mid \F_{i-1})\leq a_i .
 \label{eq:conditional}
\end{equation}

Equation~\eqref{eq:conditional} says what the certificate needs to guarantee: after accounting
for the information used to select the request, its harm probability is no greater than $a_i$.
A score from a learned monitor does not by itself satisfy this condition.  The root begins with
risk escrow $\delta\in(0,1)$, and the governor keeps $\sum_i a_i\leq\delta$ almost surely.  It approves the
request and debits the escrow in one transaction before exposure.  A rejected request, or one
cancelled before exposure, incurs no debit.  A retry after exposure is a new activation and needs
a new certificate.

\subsection{Decision objective}

A policy $\pi$ decides how many sandbox candidates to generate, which branches to activate, what
scope to grant, and when to ask a human or stop.  Let $W(\pi)$ denote expected useful task value
and $C(\pi)$ denote expected compute and latency cost in the same utility units.  The policy trades
task value against those costs while keeping the chance of any catastrophe in
the episode below $\delta$:

\[
 \max_{\pi}\; W(\pi)-C(\pi)
 \quad\text{s.t.}\quad
 \Prb_{\pi}\!\left(\bigcup_{i\geq1}H_i\right)\leq\delta .
\]

This trajectory-level chance constraint is hard to check online.  The escrow provides a
conservative condition that the governor can enforce as the tree grows.  Sandbox computation
still uses resources, which are handled by separate compute and concurrency budgets.

\section{Analysis}

\subsection{Anytime control under adaptive spawning}

The accounting idea is to assign each selected request a conditional probability allowance and
keep total spending within $\delta$.  Adaptive spawning may change which requests reach the gate;
the ledger still bounds their combined allowance.

\begin{theorem}[Selection-valid tree guarantee]
For any finite or countable sandbox tree generated by an adaptive policy, suppose each approved
activation satisfies Eq.~\eqref{eq:conditional} and
$\sum_i a_i\leq\delta$ almost surely.  Then
\[
 \Prb\!\left(\bigcup_{i\geq1}H_i\right)\leq\delta .
\]
The conclusion does not require independent branches or a predetermined depth limit.
\end{theorem}

\begin{proof}[Proof sketch]
For every finite $n$, Boole's inequality, the tower property, and conditional validity give
\[
\Prb\!\left(\bigcup_{i=1}^{n}H_i\right)
\leq\sum_{i=1}^{n}\E[\Prb(H_i\mid\F_{i-1})]
\leq\E\!\left[\sum_{i=1}^{n}a_i\right]
\leq\delta.
\]
Letting $n\to\infty$ and applying monotone convergence gives the stated bound for the full
countable sequence.
\end{proof}

The escrow uses predictable probability spending: each allowance is chosen from information
available before the corresponding activation.  Here, the allowance is chosen just before
irreversible authority is granted.  This is related to online alpha-spending
\cite{weinstein2020online}.  The bound concerns the sequence of approved activations; the i.i.d.
branching model below provides structural insight.

The root balance need not stay centralized.  It may be passed down the tree as long as delegation
cannot copy it.  Suppose a parent $v$ holds escrow $B_v$.  It may spend $s_v$ on its own
activations and transfer balances $B_u$ to authority-bearing children when
$s_v+\sum_u B_u\leq B_v$.  On any finite ancestor-closed subtree, repeated use of this inequality
bounds its debits plus the nonnegative outgoing balances by $\delta$.  Taking an increasing union
over a countable tree shows that total debits remain at most $\delta$.  Sandboxed children receive
compute tokens but no catastrophe escrow as long as the sandbox assumption holds.

\subsection{Why risk should vest at activation}

Spawn charging uses risk allowance before the parent knows whether a branch will be useful or
activated.  To compare the two accounting rules on equal terms, assign each spawned branch a
nonnegative reservation $b_v$ before observing its sandbox evidence.  \emph{Spawn charging}
irrevocably reserves $b_v$ when the branch is created and does not reclaim it if the branch is
discarded.  Progressive vesting applies the same charge when
the gate activates $v$.  If the eventual certificate depends on later evidence, let $b_v$
upper-bound any charge that the gate could assign.

\begin{proposition}[Option-value dominance]
Fix the same sandbox, activation gate, compute constraints, and branch charges $b_v$.
Every policy feasible under spawn charging is feasible under progressive vesting with identical
activated actions and harm distribution.  The optimal expected utility under
progressive vesting is weakly higher.
\end{proposition}

Activated branches are a subset of spawned branches; some spawned branches may never be
activated.  Since $A\subseteq S$ on every path,
$\sum_{v\in A}b_v\leq\sum_{v\in S}b_v$.  The numerical study below gives one case with strict
improvement.  Spawned branches still consume compute; the proposition separates that cost from
catastrophe-risk allowance when external controls enforce the sandbox boundary.

Selection creates a statistical concern.  Suppose half of candidate branches have harm
probability zero and half have harm probability $2r$, with $r\leq1/2$.  A monitor that reports
the marginal risk $r$ is calibrated across candidates before filtering.  If the parent observes
a signal identifying the second group and activates those branches, the conditional risk among
activated branches becomes $2r$.  A recursive system can repeat this filtering at every level.
A marginal certificate---or one calibrated for a fixed pipeline and reused after endogenous
routing---may then fail after selection.  Equation~\eqref{eq:conditional} conditions on the
information used to choose the branch.

\subsection{The authority-reproduction phase transition}

The relevant reproduction rate is the expected number of children that inherit authority, rather
than the total number of branches created.  We capture this distinction with a marked branching process.  Let $m$ be
the mean number of sandbox candidates generated by each authority node, and let $s$ be the
probability that a candidate receives authority.  Each authority-bearing agent generates a
random number $N$ of sandbox candidates; write $\phi$ for the probability generating function of
$N$, whose mean is $m$.  Each candidate independently receives authority with probability $s$.
The count $K$ of authority-bearing children has generating function

\[
 \psi(z)=\phi(1-s+sz),\qquad \RA=\E[K]=ms.
\]

Only authority-bearing children continue the risky chain, so their mean count is
$\RA=ms$.  Independently of its offspring count and of all other nodes, each authority-bearing
node causes a local catastrophe with probability $p$.  Let $R(p)$ be the probability that at least one catastrophe occurs anywhere in
the unbounded authority tree, and let $h(p)=1-R(p)$.  These independence assumptions apply only to
this stylized analysis, not to the anytime theorem.

\begin{theorem}[Branching phase transition]
The no-harm probability is the largest fixed point in $[0,1]$ satisfying
\begin{equation}
 h=(1-p)\psi(h).
 \label{eq:fixedpoint}
\end{equation}
\begin{enumerate}
 \item If $\RA<1$ and the second factorial moment $\beta=\psi''(1)$ is finite, then
 $R(p)=p/(1-\RA)+O(p^2)$ as $p\downarrow0$.
 \item If $\RA=1$, the offspring law is nondegenerate, and
 $0<\beta=\psi''(1)<\infty$, then
 $R(p)\sim\sqrt{2p/\beta}$.
 \item If $\RA>1$ and $\xi<1$ is the extinction probability, then
 $\lim_{p\downarrow0}R(p)=1-\xi$.
\end{enumerate}
\end{theorem}

\begin{proof}[Proof sketch]
No harm through depth $L+1$ requires the root to be harmless and every child subtree to be
harmless through depth $L$, which gives the recursion in Eq.~\eqref{eq:fixedpoint}.  Iterating this
recursion from $1$ produces a decreasing sequence whose limit is the largest fixed point.  Writing $x=1-h$, the
finite-$\beta$ expansion
$\psi(1-x)=1-\RA x+(\beta/2)x^2+o(x^2)$ applies.  Substituting this expansion into the fixed-point
equation and comparing leading terms gives the first two asymptotic regimes.  For the third
regime, an extinct tree is finite and becomes harmless with probability tending to one as
$p\downarrow0$.  On nonextinction, infinitely many independent local harm events make eventual harm
almost sure.
\end{proof}

The threshold has a natural branching interpretation.  When $\RA<1$, authority lineages tend to
die out; at $\RA=1$, they are balanced.  When $\RA>1$, a lineage may survive indefinitely, creating
a risk floor that smaller local failure probabilities cannot remove.

\begin{corollary}[Broad search, subcritical authority]
The mean candidate count $m$ may exceed one while the authority tree remains subcritical if
$ms<1$.  When $\RA<1$, the expected total number of authority nodes, including the root, is
$1/(1-\RA)$.  In the untruncated benchmark, suppose each authority node has a constant valid
charge $0<r\leq\delta$.  The ex ante union bound is then at most $r/(1-\RA)$; requiring it to be
no more than $\delta$ is equivalent to $\RA\leq1-r/\delta$.
\end{corollary}

This calculation describes expected exposure before an episode begins.  A realized tree may be
much larger than its expectation, so the calculation does not replace pathwise escrow.  In the
unbounded homogeneous model, trajectory harm can vanish with local harm only when authority
reproduction is not supercritical.  Pathwise escrow can enforce the $\delta$ bound in any branching
regime.  Finite depth removes the infinite-tree survival floor, but does not by itself certify a
prescribed risk level.  Dependence may also create a floor that persists as local i.i.d. risk falls.
For example, suppose a shared defect occurs with probability $\gamma$ and causes catastrophe with
probability one.  Conditional on no defect, assume the i.i.d. branching model holds.  The resulting
risk is $\gamma+(1-\gamma)R(p)\geq\gamma$.

\subsection{Risk and compute shadow prices}

The single-type model treats every authority node alike.  Real systems may have planning and
deployment roles with different costs and continuation paths.  To represent these differences, let
$i\in\{1,\ldots,d\}$ index authority-node type and $a\in\mathcal A_i$ index a mode available to
that type.  A mode has useful contribution $w_{ia}$, compute cost $c_{ia}$, and activation charge
$r_{ia}$.  It also produces an expected number $M_{ia,j}$ of type-$j$ authority children.  The
nonnegative vector $\mu$ records the root population.

We make a stability assumption: for some $\eta>0$, every admissible stationary policy has an
offspring matrix with spectral radius at most $1-\eta$.  This keeps the expected authority
population finite.  Let $y_{ia}$ record the expected number of type-$i$ authority nodes that use
mode $a$.  Within this finite-type stationary model, planning becomes the following
occupancy-measure linear program:

\begin{equation}
\begin{aligned}
 \max_{y\geq0}\quad &\sum_{i,a}y_{ia}w_{ia}\\
 \text{s.t.}\quad
 &\sum_a y_{ja}=\mu_j+\sum_{i,a}y_{ia}M_{ia,j} &&\forall j,\\
 &\sum_{i,a}y_{ia}r_{ia}\leq\delta,\qquad
 \sum_{i,a}y_{ia}c_{ia}\leq\bar C .
\end{aligned}
\label{eq:lp}
\end{equation}

The first constraint balances the expected flow of nodes into and out of each type.  The other
two keep expected cumulative risk charges and compute within their respective budgets.  The dual
representation turns these shared budgets into prices that can guide local mode choices.

\begin{theorem}[Decentralized shadow prices]
Assume Eq.~\eqref{eq:lp} is feasible, the type and mode sets are finite, costs and charges are
nonnegative, and uniform subcriticality holds.  Then the LP is equivalent to optimization over
stationary randomized policies, and strong duality holds.  There exist risk and compute prices
$\lambda,\nu\geq0$ and continuation values $V_i$ such that every mode satisfies
\[
 V_i\geq w_{ia}-\lambda r_{ia}-\nu c_{ia}+\sum_jM_{ia,j}V_j ,
\]
and equality holds for every mode with $y_{ia}>0$.  For modes indexed by integer fanout $k$,
define $G_{ik}=w_{ik}-\nu c_{ik}+\sum_jM_{ik,j}V_j$.  Suppose the available fanout levels are
consecutive and nested, $r_{ik}-r_{i,k-1}>0$, and the ratios
$(G_{ik}-G_{i,k-1})/(r_{ik}-r_{i,k-1})$ decrease with $k$.  A maximizing fanout can then be chosen
as the largest $k$ whose marginal ratio is at least $\lambda$.  If no positive increment clears
the threshold, choose the smallest available fanout.  Once the dual values are fixed, fanout
follows a threshold rule analogous to a base-stock policy.
\end{theorem}

The two prices have a practical reading.  The compute component $c_{ia}$, which may include
sandbox search initiated by the mode, is priced by $\nu$.  Its activation charge $r_{ia}$ is also
priced by $\lambda$.  At values of
$\delta$ where the optimum is differentiable, a smaller risk budget weakly increases the risk
price.  With decreasing marginal ratios, the policy then removes the least valuable authority
increments first.

\begin{corollary}[LP certificate]
If $r_{ia}$ is a valid harm bound conditional on the full pre-activation history whenever mode
$(i,a)$ is activated, the stationary policy induced by any feasible occupancy $y$ satisfies
\[
 \Prb(\text{any harm})\leq \E\!\sum_{\text{activated }v}r_v
 =\sum_{i,a}y_{ia}r_{ia}\leq\delta.
\]
\end{corollary}

The first inequality uses the same conditional union-bound argument, and the equality follows
from expected occupancy.  The LP caps expected cumulative charge across episodes, which gives
the unconditional episode-level probability bound above.  The governor adds the stronger
requirement that the escrow limit hold on each realized trajectory.

\section{Progressive Risk Vesting Governor}

The governor sits between sandboxed branches and actions or capabilities that cross the
application-defined activation boundary.
Algorithm~\ref{alg:prv} first applies hard checks on permissions and destinations, then evaluates
the remaining requests probabilistically.  A branch cannot certify its own request or alter the
escrow service.  Before any debit, the governor immutably binds the requester identity and lineage
to the requested action, arguments, and capability scope.

\begin{algorithm}[t]
\caption{Progressive risk vesting at the capability-activation boundary}
\label{alg:prv}
\begin{algorithmic}[1]
\Require root allowance $\delta$; independent compute budget
\State launch root and descendants with sandbox-only capabilities
\For{each activation request $i$ in atomic admission order}
  \State immutably bind request details to its requester and lineage
  \If{a deterministic policy rule rejects the request}
    \State deny the request or ask a human to narrow it
  \Else
    \State compute selection-conditional certificate $a_i$
    \If{$a_i$ exceeds uncommitted lineage escrow}
      \State deny the request or request a narrower revision
    \Else
      \State atomically debit $a_i$ and authorize one execution
      \State keep debit after exposure; gate each retry again
    \EndIf
  \EndIf
\EndFor
\end{algorithmic}
\end{algorithm}

Human approval produces a new request that re-enters the gate; it does not skip certification or
debit.  Each successful pass authorizes one action with fixed arguments and scope.  A longer
capability lease needs a certificate that covers all uses and downstream effects during the lease.

Local analysis and simulation can remain sandboxed when they do not cross the application-defined
activation boundary.  In coding, edits and tests can stay inside an isolated copy; secret access
and deployment require escrow.  In banking, drafting a transfer can remain sandboxed; execution
requires escrow.  The boundary varies by application, so the surrounding system needs to enforce
and log it.

\section{Numerical Results}

We use two reproducible synthetic studies to illustrate the analytical predictions across
parameter settings.  These studies do not estimate safety in deployed agents.  The accompanying
code and data record the fixed seed and parameter settings.

\subsection{Phase transition}

When candidate counts are Poisson, thinning gives Poisson authority offspring with mean
$\RA=ms$.  We solved Eq.~\eqref{eq:fixedpoint} to numerical tolerance $10^{-14}$ over
$m\in[0.4,3]$ and $s\in[0.05,1]$.  Figure~\ref{fig:phase}a fixes $p=0.005$.  The transition occurs
at $ms=1$.  A system that generates three sandbox candidates per authority node remains
subcritical if fewer than one third are promoted on average.

Figure~\ref{fig:phase}b shows how trajectory harm varies with local harm.  At $p=0.005$, the
trajectory-harm probability is $1.23\%$ for $\RA=0.6$, $9.68\%$ at criticality, and $51.86\%$
for $\RA=1.4$.  In the supercritical case, the nonextinction floor is $51.10\%$: reducing local
harm from $0.005$ to $10^{-5}$ leaves trajectory harm at $51.10\%$.  For comparison, the
subcritical value at $10^{-5}$ is $2.50\times10^{-5}$, matching the first-order multiplier
$1/(1-0.6)=2.5$.

\begin{figure*}[t]
  \centering
  \includegraphics[width=\textwidth]{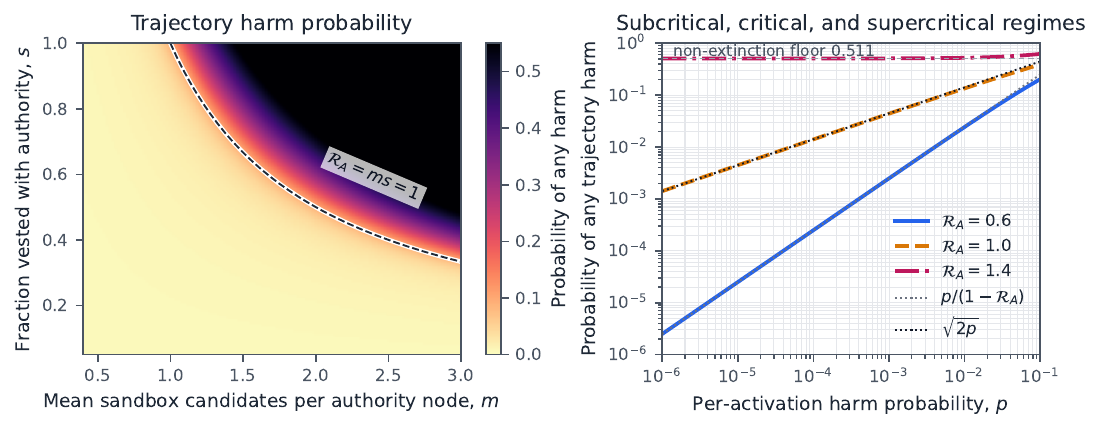}
  \Description{Two-panel numerical study. The left panel is a heatmap of trajectory-harm
  probability as a function of mean sandbox fanout and authority-activation probability, with the
  critical boundary marked. The right panel is a log-log plot of harm probability in the
  subcritical, critical, and supercritical regimes.}
  \caption{Numerical calculations for Poisson branching.  (a) Probability of any trajectory harm at
  per-activation harm $p=0.005$; the dashed curve marks the authority critical boundary $ms=1$.
  (b) Computed harm probability as a function of $p$ in three reproduction regimes.  The dotted
  lines show the subcritical and critical asymptotics.}
  \label{fig:phase}
\end{figure*}

\subsection{Option value of delayed vesting}

The second study examines how delayed vesting changes candidate selection.  We simulate 200,000
episodes (seed 20260901) and divide them equally between tuning and held-out evaluation.  Each
episode contains $n$ sandbox candidates.  Candidate quality follows
$Q\sim\mathrm{Beta}(2,2)$, and the observed score is $Q+\epsilon$, where
$\epsilon\sim\mathcal N(0,0.15^2)$, independently across candidates.  The parent activates the
candidate with the highest score.  Net utility is the selected candidate's true quality minus
$0.004n$.  The trajectory budget is $0.05$, and the valid charge for the single activation is
$0.01$.

The tuning set selects $n=5$ under spawn charging and $n=14$ under progressive vesting.  On the
held-out episodes, mean net utility is $0.6971$ under spawn charging (standard error $0.00048$)
and $0.7372$ under progressive vesting (standard error $0.00037$).  The paired difference is
$0.0401$, or $5.75\%$.  Its 95\% confidence interval is $[0.0391,0.0411]$.
Figure~\ref{fig:option} shows the full held-out curve.  The experiment does not imply that fourteen
candidates are optimal in an LLM deployment.  Here, delayed charging leaves more candidate counts
feasible when pruning is informative and sandbox compute is charged separately.

\begin{figure}[t]
  \centering
  \includegraphics[width=\columnwidth]{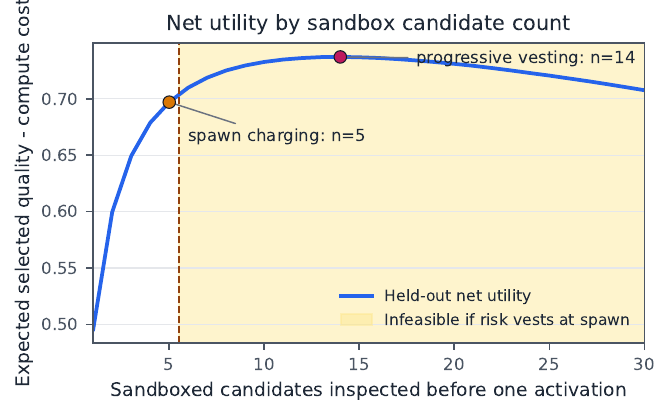}
  \Description{Line chart of held-out expected net utility as a function of sandbox candidate
  count. Spawn charging becomes infeasible above five candidates, and progressive vesting
  reaches its maximum at fourteen candidates.}
  \caption{Held-out candidate-selection results.  Spawn charging makes $n>5$ infeasible because
  it reserves $0.01$ for every candidate.  Progressive vesting debits $0.01$ for the selected
  candidate, so adding sandbox candidates does not increase the activation charge.}
  \label{fig:option}
\end{figure}

\section{Scope, Limitations, and Evaluation Roadmap}

\paragraph{Conditional certification.}
The anytime theorem takes valid certificates $a_i$ as inputs; it does not show how to learn them.
A learned gate would need calibration and auditing under the routing and escalation policy used
in deployment.  Calibration within fixed strata or anytime-valid upper confidence sequences may
provide useful starting points.  Distribution shift or a strategic attack may break
either approach.  The guarantee does not apply if the gate is compromised.

\paragraph{Sandbox enforcement.}
The model assigns no catastrophe allowance to sandboxed reasoning under the assumption that
external controls block routes to the specified harm.  Access to sensitive data or the ability
to modify shared state may cross this boundary.  Resource exhaustion is a separate concern.  Even
when it avoids the specified harm, a supercritical sandbox process can consume unbounded compute
on nonextinction.  Separate compute and concurrency limits remain needed.

\paragraph{Branching model.}
Real agent trees may have neither i.i.d. offspring nor independent local failures.  Shared models
or tools can cause several branches to fail together.  The phase-transition model is best read as a
diagnostic rather than a literal description of deployed agents.  The conditional escrow result
allows dependent branches when its certificate assumption holds.  The numerical results are
synthetic and do not establish safety for WebSwarm or any deployed system.

\paragraph{Evaluation.}
A deployment study could use disposable coding environments and AgentDojo banking tasks
\cite{debenedetti2024agentdojo}.  It should compare \prv{} with unrestricted spawning and spawn
charging under matched compute limits.  The environments should record each gate decision and
capability exposure.  Primary outcomes should be catastrophic-action rate and useful task completion,
with confidence intervals for the safety estimates.  Tree size and post-selection calibration
would provide secondary diagnostics.  Held-out tests could vary the attacks or model families.

\paragraph{Ethics.}
This approach could be used to scale autonomous-agent fleets without enough oversight.
Experiments should keep irreversible actions inside emulators and avoid live credentials.  They
should report both useful autonomy and residual risk.  The human operator remains
responsible for the catastrophe definition and permission policy, as well as the deployment
decision.

\section{Conclusion}

In recursive agents, creating a branch and granting it irreversible authority are separate
decisions.  Flat orchestration can make this distinction less visible.  Progressive risk vesting
allows broad sandbox search; the shared risk allowance is debited at activation rather than spawn.  Under the stated
conditional-certificate assumption, the escrow bounds the probability of a catastrophic
activation in an adaptively generated tree with dependent branches.  In the branching benchmark,
the authority reproduction rate marks a qualitative phase transition.  Below one, trajectory
harm is linear in local harm.  At one, it scales with the square root of local harm.  Above one, a
nonzero floor remains.  For a uniformly subcritical, finite-type stationary model, shadow prices
guide which branches receive compute and authority.  This suggests spawning candidates to gather
information, then vesting authority only when the risk-adjusted continuation value justifies a
debit from the remaining escrow.

\appendix
\section{Full Proofs}

\subsection{Selection-valid tree guarantee}

For finite $n$, Boole's inequality and the tower property give:
\[
 \Prb\!\left(\bigcup_{i=1}^{n}H_i\right)
 \leq\sum_{i=1}^{n}\Prb(H_i)
 =\sum_{i=1}^{n}\E\!\left[\Prb(H_i\mid\F_{i-1})\right]
 \leq\E\!\left[\sum_{i=1}^{n}a_i\right]
 \leq\delta.
\]
Continuity from below extends the result to $n\to\infty$.  Arbitrary branch dependence is allowed
provided the conditional certificates remain valid.

To check the ledger invariant, fix a finite set of activations and take its finite ancestor
closure.  Treat balances transferred beyond this subtree as nonnegative boundary balances.
Repeated use of the local conservation inequality bounds the selected debits plus those boundary
balances by $\delta$.  Dropping the boundary balances and exhausting the countable activation
sequence by finite sets proves the claim.

\subsection{Option-value dominance}

Let $S(\pi)$ and $A(\pi)$ denote the spawned and activated node sets under policy $\pi$, with
$A(\pi)\subseteq S(\pi)$ on every path.  Give branch $v$ the same nonnegative charge $b_v$ under
both rules.  Spawn charging irrevocably reserves $b_v$ when branch $v$ is created and does not
reclaim that reservation after the branch is discarded, so
$\sum_{v\in S(\pi)}b_v\leq\delta$.  Applying the same sandbox and activation decisions under
progressive vesting instead spends $\sum_{v\in A(\pi)}b_v$, which is no larger.  Every policy
feasible under spawn charging is also feasible under progressive vesting, proving weak
dominance.  Strict improvement requires further conditions; the candidate-selection experiment
provides one concrete case.

\subsection{Branching phase transition}

Let $h_L$ denote the probability of no harm through depth $L$, with $h_{-1}=1$.  Conditioning on
the root and its offspring gives $h_L=(1-p)\psi(h_{L-1})$.  This sequence is decreasing and bounded
below, so it converges to a fixed point $h$.  Because the recursion is monotone and begins at one,
the limit is the largest fixed point.

Write $x=1-h$.  As $p\downarrow0$,
\[
1-x=(1-p)\left(1-\RA x+\frac{\beta}{2}x^2+o(x^2)\right).
\]
If $\RA<1$, the implicit-function theorem applies at $(p,x)=(0,0)$ and gives
$(1-\RA)x=p+O(p^2)$, so $x=p/(1-\RA)+O(p^2)$.  If $\RA=1$, the linear term cancels and
$p=(\beta/2)x^2+o(p+x^2)$, so $x\sim\sqrt{2p/\beta}$.

If $\RA>1$, let $\xi<1$ denote the extinction probability.  On extinction, the tree is finite
almost surely, so bounded convergence gives a conditional no-harm probability approaching one as
$p\downarrow0$.  On nonextinction, infinitely many independent local harm events imply harm
almost surely for every $p>0$.  It follows that $h(p)\to\xi$ and $R(p)\to1-\xi$.

\subsection{Occupancy linear program and duality}

For a stationary policy $\pi$, define
$M_\pi(i,j)=\sum_a\pi_i(a)M_{ia,j}$.  Uniform subcriticality ensures that
$(I-M_\pi)^{-1}$ is finite.  If $x_i$ denotes expected type-$i$ occupancy, then
$x=\mu+xM_\pi$, and $y_{ia}=x_i\pi_i(a)$ satisfies the flow equations.  To recover a policy from
feasible nonnegative $y$, set $x_i=\sum_a y_{ia}$ and
$\pi_i(a)=y_{ia}/x_i$ whenever $x_i>0$.  The policy at unreachable types may be chosen
arbitrarily.  The flow equations recover the same occupancy, establishing policy--occupancy
equivalence.

Introduce a free dual variable $V_j$ for each flow equality and nonnegative multipliers $\lambda$
and $\nu$ for the risk and compute constraints.  The resulting dual is
\[
 \min_{V,\lambda,\nu}\;\mu\!\cdot\!V+\lambda\delta+\nu\bar C
\]
subject to
\[
 V_i\geq w_{ia}-\lambda r_{ia}-\nu c_{ia}+\sum_jM_{ia,j}V_j
 \quad\text{for all }(i,a).
\]
Standard finite-dimensional LP duality equates the primal and dual optima.  Complementary
slackness makes the constraint tight for every mode with $y_{ia}>0$.  For a type $i$ with nested,
consecutive fanout modes, let $D_{ik}=G_{ik}-\lambda r_{ik}$.  Then
\[
D_{ik}-D_{i,k-1}
=(r_{ik}-r_{i,k-1})
\left[
\frac{G_{ik}-G_{i,k-1}}{r_{ik}-r_{i,k-1}}-\lambda
\right].
\]
Positive risk increments and decreasing ratios make these adjacent differences change sign at
most once.  A maximizing fanout therefore follows the stated threshold rule, with ties resolved
toward the largest maximizing level.

If each $r_{ia}$ is a valid harm bound conditional on the full pre-activation history whenever
mode $(i,a)$ is activated, the same tower-property and union-bound argument as in the first proof yields
\[
 \Prb(\text{any harm})\leq\E\!\sum_{\text{activated }v}r_v.
\]
By the definition of occupancy, the right-hand side is
$\sum_{i,a}y_{ia}r_{ia}\leq\delta$ for any feasible solution of the linear program.

\section{Numerical Protocol}

The script \texttt{simulate.py} uses Python, NumPy, and Matplotlib and fixes the random seed at
20260901.  It reproduces the plots from the accompanying CSV files.

\paragraph{Fixed point.}
For Poisson offspring, the probability-generating function is
\[
 \psi(z)=\exp(\RA(z-1)).
\]
We initialize fixed-point iteration at $h_0=1$ and stop when consecutive iterates differ by at
most $10^{-14}$.  The grid contains 131 values of $m$ from 0.4 to 3.0 and 120 values of $s$ from
0.05 to 1.0.  The heatmap uses local harm 0.005.  The asymptotic panel evaluates 121
logarithmically spaced values of $p$ from $10^{-6}$ to $10^{-1}$.

\paragraph{Candidate selection.}
We generate 200,000 episodes with 30 candidates each.  For every $n=1,\ldots,30$, we use the
first $n$ candidates from the same episode, which reduces Monte Carlo noise across values of
$n$.  We use the first 100,000 episodes to select the best feasible $n$ under each accounting
rule and the remaining 100,000 to evaluate those choices.  Each episode contains 30 independent
$\mathrm{Beta}(2,2)$ qualities, each paired with an independent Gaussian score error with standard
deviation 0.15.  The
parent selects the candidate with the highest noisy score among the first $n$ candidates.  Net
utility is the selected latent quality minus $0.004n$.  We compute the standard errors and paired
confidence interval from the held-out episodes.  With root allowance 0.05 and charge 0.01, spawn charging permits
$n\leq5$.  Progressive vesting activates the selected candidate and spends 0.01 for every $n$ in
the sweep.

\section*{AI-Assistance Disclosure}

The author used OpenAI Codex to assist with the research framing and formalization, proof
development, the synthetic numerical study, drafting, and editing.  The author reviewed the
manuscript and remains responsible for its proofs, numerical results, citations, and claims.

\balance
\bibliographystyle{ACM-Reference-Format}
\bibliography{references}

\end{document}